\documentclass[11pt]{article}

\usepackage[utf8]{inputenc}
\usepackage[backref,colorlinks,citecolor=blue,bookmarks=true]{hyperref}
\usepackage{amsmath}
\usepackage{amssymb}
\usepackage{amsthm}
\usepackage{bm}
\usepackage{bbm}
\usepackage{dsfont}
\usepackage[affil-it]{authblk}
\usepackage{comment}
\usepackage{mathtools}
\usepackage[shortlabels]{enumitem}
\usepackage{complexity}
\usepackage[capitalize]{cleveref}
\crefname{ineq}{inequality}{inequalities}
\creflabelformat{ineq}{#2{\upshape(#1)}#3}
\usepackage{nag}
\usepackage{xspace}
\usepackage{subcaption}
\usepackage{float}

\usepackage[margin=1in]{geometry}
\theoremstyle{plain}
\newtheorem{theorem}{Theorem}[section]
\newtheorem{lemma}[theorem]{Lemma}
\newtheorem{corollary}[theorem]{Corollary}

\newtheorem{assumption}[theorem]{Assumption}

\theoremstyle{definition}
\newtheorem{definition}[theorem]{Definition}

\theoremstyle{remark}
\newtheorem{remark}[theorem]{Remark}

\newcommand*{\N}{{\mathbb{N}}}

\let\R\relax
\newcommand*{\R}{{\mathbb{R}}}

\newcommand*{\cC}{{\mathcal{C}}}

\newcommand*{\cCo}{{\mathcal{C}_\text{orth}}}

\newcommand*{\nher}{\tilde{H}}

\let\Pr\relax
\newcommand*{\Pr}{\mathbb{P}}
\let\Ex\relax
\newcommand*{\Ex}{\mathbb{E}}
\newcommand{\Ind}{\mathbbm{1}}
\DeclarePairedDelimiter{\inn}{\langle}{\rangle}
\DeclareMathOperator{\sda}{SDA}
\DeclareMathOperator{\sgn}{sign}
\DeclareMathOperator{\relu}{ReLU}
\newcommand*{\sigmoid}{\sigma}
\DeclareMathOperator{\unif}{Unif}
\renewcommand{\th}{^\text{th}}
\let\hat\widehat

\let\poly\relax
\DeclareMathOperator{\poly}{poly}

\newcommand{\abs}[1]{\left\vert {#1} \right\vert}

\newcommand\norm[1]{\left\lVert#1\right\rVert}

\newcommand{\dotp}[2]{#1~{\cdot}~#2}

\title{Superpolynomial Lower Bounds for Learning One-Layer Neural Networks using Gradient Descent}
\author[1]{Surbhi Goel}
\author[1]{Aravind Gollakota}
\author[2]{Zhihan Jin}
\author[1]{Sushrut Karmalkar}
\author[1]{Adam Klivans}
\affil[1]{Department of Computer Science, University of Texas at Austin}
\affil[2]{Department of Computer Science, Shanghai Jiao Tong University}
\date{June 22, 2020}

\begin{document}
	
	\maketitle
	
	\begin{abstract}
		We prove the first superpolynomial lower bounds for learning one-layer neural networks with respect to the Gaussian distribution using gradient descent.  We show that any classifier trained using gradient descent with respect to square-loss will fail to achieve small test error in polynomial time given access to samples labeled by a one-layer neural network.    For classification, we give a stronger result, namely that any statistical query (SQ) algorithm (including gradient descent) will fail to achieve small test error in polynomial time. Prior work held only for gradient descent run with small batch sizes, required sharp activations, and applied to specific classes of queries.  Our lower bounds hold for broad classes of activations including ReLU and sigmoid.  The core of our result relies on a novel construction of a simple family of neural networks that are exactly orthogonal with respect to all spherically symmetric distributions.
	\end{abstract}
	
	\section{Introduction}
	
	%Backpropagation and gradient descent are at the core of modern deep learning systems.  
	A major challenge in the theory of deep learning is to understand when gradient descent can efficiently learn simple families of neural networks. The associated optimization problem is nonconvex and well known to be computationally intractable in the worst case.  For example, cyphertexts from public-key cryptosystems can be encoded into a training set labeled by simple neural networks \cite{klivans2009cryptographic}, implying that the corresponding learning problem is as hard as breaking cryptographic primitives.   These hardness results, however, rely on discrete representations and produce relatively unrealistic joint distributions.
	%(we discuss other hardness results at length in Section XX).    
	
	\paragraph{Our Results.} In this paper we give the first superpolynomial lower bounds for learning neural networks using gradient descent in arguably the simplest possible setting: we assume the marginal distribution is a spherical Gaussian, the labels are noiseless and are exactly equal to the output of a one-layer neural network (a linear combination of say ReLU or sigmoid activations), and the goal is to output a classifier whose test error (measured by square-loss) is small.  We prove---unconditionally---that gradient descent fails to produce a classifier with small square-loss if it is required to run in polynomial time in the dimension.  Our lower bound depends only on the algorithm used (gradient descent) and not on the architecture of the underlying classifier.  That is, our results imply that current popular heuristics such as running gradient descent on an overparameterized network (for example, working in the NTK regime \cite{JacotHG18}) will require superpolynomial time to achieve small test error. 
	
	\paragraph{Statistical Queries.} We prove our lower bounds in the now well-studied statistical query (SQ) model of \cite{kearns1998efficient} that captures most learning algorithms used in practice.  
	%including gradient descent. %due to that captures most PAC learning problems. 
	For a loss function $\ell$ and a hypothesis $h_\theta$ parameterized by $\theta$, the true population loss with respect to joint distribution $D$ on $X \times Y$ is given by $\Ex_{(x, y) \sim D}[\ell(h_\theta(x), y)]$, and the gradient with respect to $\theta$ is given by $\Ex_{(x, y) \sim D}[\ell'(h_\theta(x), y) \nabla_\theta h_\theta(x)]$.  
    In the SQ model, we specify a query function $\phi(x,y)$ and receive an estimate of $|\Ex_{(x,y) \sim D}[\phi(x,y)]|$ to within some tolerance parameter $\tau$.  An important special class of queries are correlational or inner-product queries, where the query function $\phi$ is defined only on $X$ and we receive an estimate of $|\Ex_{(x,y) \sim D}[\phi(x)\cdot y]|$ within some tolerance $\tau$.  It is not difficult to see that (1) the gradient of a population loss can be approximated to within $\tau$ using statistical queries of tolerance $\tau$ and (2) for square-loss only inner-product queries are required.
    
	Since the convergence analysis of gradient descent holds given sufficiently strong approximations of the gradient, lower bounds for learning in the SQ model \cite{kearns1998efficient, blum1994weakly, szorenyi2009characterizing, feldman2012complete, feldman2017general} directly imply unconditional lower bounds on the running time for gradient descent to achieve small error.
    %, we can view gradient descent as a statistical query algorithm   The reason this view is particularly useful is because we can establish strong, unconditional and considerably general lower bounds for SQ learning algorithms \cite{kearns1998efficient, blum1994weakly, szorenyi2009characterizing, feldman2012complete, feldman2017general}. In the context of neural networks, there has been prior work \cite{song2017complexity, vempala2018gradient} that has used the SQ framework to give lower bounds in various settings (see Related Work).
	%\paragraph{Our contributions.}
	%We strengthen prior SQ lower bounds for one hidden layer neural networks for a large class of input distributions. Our first main result may be stated slightly informally as follows. An inner product, or correlational, query is one of the form $h(x, y) = g(x) \cdot y$.
    We give the first superpolynomial lower bounds for learning one-layer networks with respect to any Gaussian distribution for any SQ algorithm that uses inner product queries:
	\begin{theorem}[informal]
		Let $\cC$ be a class of real-valued concepts defined by one-layer single-output neural networks with input dimension $n$ and $m$ hidden units (ReLU or sigmoid); i.e., functions of the form $f(x) = \sum_{i=1}^{m} a_i \sigma(w_i \cdot x)$. %and the output node being any odd real-valued activation function. 
		Then learning $\cC$ under the standard Gaussian $\mathcal{N}(0, I_n)$ in the SQ model with inner-product queries requires $n^{\Omega(\log m)}$ queries for any tolerance $\tau = n^{-\Omega(\log m)}$.
	\end{theorem}

	In particular, this rules out any approach for learning one-layer neural networks in polynomial-time that performs gradient descent on any polynomial-size classifier with respect to square-loss or logistic loss.  For classification, we obtain significantly stronger results and rule out general SQ algorithms that run in polynomial-time (e.g., gradient descent with respect to any polynomial-size classifier and any polynomial-time computable loss).  In this setting, our labels are $\{\pm 1\}$ and correspond to the {\em softmax} of an unknown one-layer neural network.  We prove the following:
	
	%By classification, we mean taking a softmax of a one-layer neural network 
	
	%Thus a direct application of the above theorem gives us the following lower bound for GD.
	%\begin{corollary}[informal]
	%	Gradient descent requires a superpolynomial number of iterations to learn one-layer neural nets with respect to squared loss.
	%\end{corollary}
	
	%While real-valued neural networks are directly suited for regression problems, neural networks are also widely used for classification problems. Here the network's output is interpreted as a label probability, so that the network constitutes a 
	
	%\emph{probabilistic} concept, or p-concept. Our second main result gives a very general lower bound for SQ learning neural networks in the p-concept model. As part of the proof, we extend the established theory of SQ lower bounds to p-concepts. To our knowledge, p-concepts had not been previously considered in the SQ model for any problem.
	
	\begin{theorem}[informal]
	Let $\cC$ be a class of real-valued concepts defined by a one-layer neural network in $n$ dimensions with $m$ hidden units (ReLU or sigmoid) feeding into any odd, real-valued output node with range $[-1,1]$. Let $D'$ be a distribution on $\R^n \times \{\pm 1\}$ such that the marginal on $\R^n$ is the standard Gaussian $\mathcal{N}(0, I_n)$, and $\Ex[Y|X] = c(X)$ for some $c \in \cC$.
	%\begin{itemize}
	For some $b,C > 0$ and $\epsilon = Cm^{-b}$, outputting a classifier $f : \R^n \to \{\pm 1\}$ with $\Pr_{(X, Y) \sim D'}[f(X)\neq Y] \leq 1/2 - \epsilon$ requires $n^{\Omega(\log m)}$ statistical queries of tolerance $n^{-\Omega(\log m)}.$
 	%\item There exists $k > 0$ and $\epsilon = \Omega(n^{-k})$ such that outputting $f$ with $\Ex[(f(x) - c(x))^2] \leq \epsilon$ in the SQ model requires $n^{\Omega(\log m)}$ queries of tolerance $\Omega(n^{-(\log m)})$.
 	%\end{itemize}
	\end{theorem}

%queries$n^P$ Then too, learning $\cC$ in the SQ model using arbitrary queries requires $n^{\Omega(\log m)}$ queries in general. 

	The above lower bound for classification rules out the commonly used approach of training a polynomial-size, real-valued neural network using gradient descent (with respect to any polynomial-time computable loss) and then taking the sign of the output of the resulting network. 
	
	\paragraph{Our techniques.}
	At the core of all SQ lower bounds is the construction of a family of functions that are pairwise approximately orthogonal with respect to the underlying marginal distribution.  Typically, these constructions embed $2^{n}$ parity functions over the discrete hypercube $\{-1,1\}^n$.  Since parity functions are perfectly orthogonal, the resulting lower bound can be quite strong.  Here we wish to give lower bounds for more natural families of distributions, namely Gaussians, and it is unclear how to embed parity.

	%Our results rely on a construction of a simple large family of pairwise orthogonal one-hidden layer networks. The existence of such a family gives a lower bound on the SQ dimension which allows us to characterizes the complexity of any SQ algorithm.
	
	Instead, we use an alternate construction.  For activation functions $\phi, \psi : \R \rightarrow \R$, define
	\begin{align*}
	f_S(x) = \psi\left(\sum_{w \in \{-1, 1\}^k} \chi(w) \phi\left(\frac{\dotp{w}{ x_S}}{\sqrt{k}}\right)\right).
	\end{align*}
	Enumerating over every $S \subseteq [n]$ of size $k$ gives a family of functions of size $n^{O(k)}$. Here $x_S$ denotes the vector of $x_i$ for $i \in S$ (typically we choose $k = \log m$ to produce a family of one-layer neural networks with $m$ hidden units). Each of the $2^k = m$ inner weight vectors are all of unit norm, and all of the $m$ outer weights have absolute value one. Note also that our construction uses activations with zero bias term.
	
	%, so that each network has $\poly(n)$-bounded weights.
	
	% \begin{equation}
	%   \E_{x \sim D} \left[ (|x_1 + x_2| - |x_1 - x_2|)(|x_2 + x_3| - |x_2 - x_3|) \right] = 0
	% \end{equation}
	
	We give a complete characterization of the class of nonlinear activations for which these functions are orthogonal. In particular, the family is orthogonal for any activation with a nonzero Hermite coefficient of degree $k$ or higher. 
	
	Apart from showing orthogonality, we must also prove that functions in these classes are nontrivial (i.e., are not exponentially close to the constant zero function).  This reduces to proving certain lower bounds on the norms of one-layer neural networks.  The analysis requires tools from Hermite and complex analysis.
	
	%For Gaussian input, we reduce the challenge to lower bounding the Hermite coefficients of the activation function and in order to do so, we use techniques from complex analysis (see Section [REF]). 
	
	\paragraph{SQ Lower Bounds for Real-Valued Functions.}
	
    Another major challenge is that our function family is real-valued as opposed to boolean.  Given an orthogonal family of (deterministic) {\em boolean} functions, it is straightforward to apply known results and obtain general SQ lower bounds for learning with respect to $0/1$ loss.  For the case of real-valued functions, the situation is considerably more complicated. For example, the class of orthogonal Hermite polynomials on $n$ variables of degree $d$ has size $n^{O(d)}$, yet there {\em is} an SQ algorithm due to \cite{andoni2014learning} that learns this class with respect to the Gaussian distribution in time $2^{O(d)}$. More recent work due to \cite{andoni2019attribute} shows that Hermite polynomials can be learned by an SQ algorithm in time polynomial in $n$ and $\log d$.
    
   As such, it is {\em impossible} to rule out general polynomial-time SQ algorithms for learning real-valued functions based solely on orthogonal function families.  Fortunately, it is not difficult to see that the SQ reductions due to \cite{szorenyi2009characterizing} hold in the real-valued setting as long as the learning algorithm uses only {\em inner-product queries} (and the norms of the functions are sufficiently large).  Since performing gradient descent with respect to square-loss or logistic loss can be implemented using inner-product queries, we obtain our first set of desired results\footnote{The algorithms of \cite{andoni2014learning} and \cite{andoni2019attribute} do not use inner-product queries.}.
   
   Still, we would like rule out {\em general} SQ algorithms for learning simple classes of neural networks.  To that end, we consider the {\em classification} problem for one-layer neural networks and output labels after performing a {\em softmax} on a one-layer network.  Concretely, consider a distribution on $\R^n \times \{-1,1\}$ where $\Ex[Y|X] = \sigma(c(X))$ for some $c \in \cC$ and $\sigma: \R \to [-1,1]$ (for example, $\sigma$ could be tanh).  We describe two goals.  The first is to estimate the conditional mean function, i.e., output a classifier $h$ such that $\Ex[(h(x) - c(x))^2] \leq \epsilon$.  The second is to directly minimize classification loss, i.e., output a boolean classifier $h$ such that $\Pr_{X,Y \sim D}[h(X) \neq  Y] \leq 1/2 - \epsilon.$  
   
   We give superpolynomial lower bounds for both of these problems in the general SQ model by making a new connection to {\em probabilistic concepts}, a learning model due to \cite{kearns1994pconcept}.  Our key theorem gives a superpolynomial SQ lower bound for the problem of {\em distinguishing} probabilistic concepts induced by our one-layer neural networks from truly random labels.  A final complication we overcome is that we must prove orthogonality and norm bounds on one-layer neural networks that have been composed with a nonlinear activation (e.g., tanh).
	
	\paragraph{SGD and Gradient Descent Plus Noise.}
	It is easy to see that our results also imply lower bounds for algorithms where the learner adds noise to the estimate of the gradient (e.g., Langevin dynamics).  On the other hand, for technical reasons, it is known that SGD is {\em not} a statistical query algorithm (because it examines training points individually) and does not fall into our framework.  That said, recent work by \cite{Abbe} shows that SGD is {\em universal} in the sense that it can encode all polynomial-time learners.  This implies that proving unconditional lower bounds for SGD would give a proof that $\P \neq \NP$.  Thus, we cannot hope to prove unconditional lower bounds on SGD (unless we can prove $\P \neq \NP$). 
	
	\paragraph{Independent Work.} 
	Independently, Diakonikolas et al. \cite{diakonikolas2020algorithms} have given stronger correlational SQ lower bounds for the same class of functions with respect to the Gaussian distribution.  Their bounds are exponential in the number of hidden units while ours is quasipolynomial.  We can plug in their result and obtain exponential general SQ lower bounds for the associated probabilistic concept using our framework.
	
	\paragraph{Related Work.}
	There is a large literature of results proving hardness results (or unconditional lower bounds in some cases) for learning various classes of neural networks \cite{blum1989training,vu1998infeasibility,klivans2009cryptographic,livni2014computational,goel2016reliably}.
	
	The most relevant prior work is due to \cite{song2017complexity}, who addressed learning one-layer neural networks under logconcave distributions using Lipschitz queries. Specifically, let $n$ be the input dimension, and let $m$ be the number of hidden $s$-Lipschitz sigmoid units.  For $m = \tilde{O}(s \sqrt{n})$,  they construct a family of neural networks such that any learner using $\lambda$-Lipschitz queries with tolerance greater than $\Omega(1/(s^2n))$ needs at least $2^{\Omega(n)}/(\lambda s^2)$ queries. 
	%Their lower bounds hold for sufficiently smooth losses including square loss. 
	
   Roughly speaking, their lower bounds hold for $\lambda$-Lipschitz queries due to the composition of their one-layer neural networks with a $\delta$-function in order make the family more ``boolean."  Because of their restriction on the tolerance parameter, they cannot rule out gradient descent with large batch sizes.  Further, the slope of the activations they require in their constructions scales inversely with the Lipschitz and tolerance parameters. 
   %(steeper activations are required to rule out queries that are less smooth and ).  
    
    To contrast with \cite{song2017complexity}, note that our lower bounds hold for any inverse-polynomial tolerance parameter (i.e., {\em will} hold for polynomially-large batch sizes), do not require a Lipschitz constraint on the queries, and use only standard $1$-Lipschitz ReLU and/or sigmoid activations (with zero bias) for the construction of the hard family.  Our lower bounds are typically quasipolynomial in the number of hidden units; improving this to an exponential lower bound is an interesting open question.  Both of our models capture square-loss and logistic loss.   
    
    In terms of techniques, \cite{song2017complexity} build an orthogonal function family using univariate, periodic ``wave'' functions.  Our construction takes a different approach, adding and subtracting activation functions with respect to overlapping ``masks.'' Finally, aside from the (black-box) use of a theorem from complex analysis, our construction and analysis are considerably simpler than the proof in \cite{song2017complexity}.
	
	%This is not, however, a true superpolynomial lower bound because of the $\poly(n)$ bound on the tolerance: it leaves open the possibility that if queries of slightly lower tolerance, say $1/m^4$ (still $1/\poly(n))$ were used, then $\poly(n)$ queries would suffice.
	
	A follow-up work \cite{vempala2018gradient} gave SQ lower bounds for learning classes of degree $d$ orthogonal polynomials in $n$ variables with respect to the uniform distribution on the unit sphere (as opposed to Gaussians) using inner product queries of bounded tolerance (roughly $1/n^{d}$).  To obtain superpolynomial lower bounds, each function in the family requires superpolynomial description length (their polynomials also take on very small values, $1/n^{d}$, with high probability). 
	
	Shamir \cite{Shamir18} (see also the related work of \cite{shalev2017failures}) proves hardness results (and lower bounds) for learning neural networks using gradient descent with respect to square-loss.
	%A closely related work is \cite{shalev2017failures}, where the most comparable result, namely an argument that the gradient might contain ``negligible information'' about the target in a concrete sense, is essentially a restatement of that of \cite{Shamir18}. Shamir's techniques also involve constructing a family of highly-uncorrelated functions, and thus bear a strong resemblance to SQ methods.
	His results are separated into two categories: (1) hardness for learning ``natural'' target families (one layer ReLU networks) or (2) lower bounds for ``natural'' input distributions (Gaussians).  We achieve lower bounds for learning problems with {\em both} natural target families and natural input distributions.  Additionally, our lower bounds hold for any nonlinear activations (as opposed to just ReLUs) and for broader classes of algorithms (SQ).
	
	Recent work due to \cite{goel2019time} gives hardness results for learning a ReLU with respect to Gaussian distributions.  Their results require the learner to output a single ReLU as its output hypothesis and require the learner to succeed in the agnostic model of learning.   \cite{KlivansK14} prove hardness results for learning a threshold function with respect to Gaussian distributions, but they also require the learner to succeed in the agnostic model. 
	Very recent work due to Daniely and Vardi \cite{vdaniely} gives hardness results for learning randomly chosen two-layer networks.  The hard distributions in their case are not Gaussians, and they require a nonlinear clipping output activation. \\
	
     \noindent \textbf{Positive Results.} Many recent works give algorithms for learning one-layer ReLU networks using gradient descent with respect to Gaussians under various assumptions \cite{Zhao,ZhangPS17,BG17,ZhangYWG19} or use tensor methods \cite{janzamin2015beating,GLM18}.  These results depend on the hidden weight vectors being sufficiently orthogonal, or the coefficients in the second layer being positive, or both. Our lower bounds explain why these types of assumptions are necessary.

	\section{Preliminaries}
	We use $[n]$ to denote the set $\{1, \dots, n\}$, and $S \subseteq_k T$ to indicate that $S$ is a $k$-element subset of $T$. We denote euclidean inner products between vectors $u$ and $v$ by $\dotp{u}{v}$. We denote the element-wise product of vectors $u$ and $v$ by $u \circ v$, that is, $u \circ v$ is the vector $(u_1 v_1, \dots, u_n v_n)$.
	
% 	\begin{definition}[Signing]
% 		For any point $x \in \R^n$ and a sign-pattern $z \in \{\pm 1\}^n$, we define $x \circ z$ to be the vector $(x_1 z_1, \dots, x_n z_n) \in \R^n$, i.e.\ the elementwise product of $x$ and $z$, to be thought of as a signing of $x$.
% 	\end{definition}
	
	Let $X$ be an arbitrary domain, and let $D$ be a distribution on $X$. Given two functions $f, g : X \to \R$, we define their $L_2$ inner product with respect to $D$ to be $\inn{f, g}_D = \Ex_D[fg]$. The corresponding $L_2$ norm is given by $\|f\|_D = \sqrt{\inn{f, f}_D} = \sqrt{\Ex_D[f^2]}$.
	
	A real-valued concept on $\R^n$ is a function $c : \R^n \to \R$. We denote the induced labeled distribution on $\R^n \times \R$, i.e.\ the distribution of $(x, c(x))$ for $x \sim D$, by $D_c$. A probabilistic concept, or $p$-concept, on $X$ is a concept that maps each point $x$ to a random $\{\pm 1\}$-valued label in such a way that $\Ex[Y|X] = c(X)$ for a fixed function $c : \R^n \to [-1, 1]$, known as the conditional mean function. %This $c$ fully specifies the $p$-concept, so we will generally identify it with the $p$-concept itself. 
	Given a distribution $D$ on the domain, we abuse $D_c$ to denote the induced labeled distribution on $X \times \{ \pm 1\}$
	such that the marginal distribution on $\R^n$ is $D$  and $\Ex[Y|X] = c(X)$ (equivalently the label is $+1$ with probability $\frac{1 + c(x)}{2}$ and $-1$ otherwise). 
	
	%It is worth noting that the $p$-concept $c_0$ that labels each point $\pm 1$ with equal probability is just given by $c_0 = 0$.
	
	\paragraph{The SQ model} A statistical query is specified by a query function $h : \R^d \times Y \to \R$. The SQ model allows access to an SQ oracle that accepts a query $h$ of specified tolerance $\tau$, and responds with a value in $[\Ex_{x,y}[h(x,y)] - \tau, \Ex_{x,y}[h(x,y)] + \tau]$.\footnote{In the SQ literature, this is referred to as the STAT oracle. A variant called VSTAT is also sometimes used, known to be equivalent up to small polynomial factors \cite{feldman2017general}.  While it makes no substantive difference to our superpolynomial lower bounds, our arguments can be extended to VSTAT as well.} To disallow arbitrary scaling, we will require that for each $y$, the function $x \mapsto h(x, y)$ has norm at most 1.
	%We emphasize that for a real-valued concept $c$, $h$ takes as input a point $x$ and its \emph{real-valued} label $c(x)$, while for a $p$-concept $c$, $h$ takes as input a point and its random \emph{boolean} label $y$ (where $\Ex[Y|X] = c(X)$). 
	In the real-valued setting, a query $h$ is called a correlational or inner product query if it is of the form $h(x, y) = g(x) \cdot y$ for some function $g$, so that $\Ex_{D_c}[h] = \Ex_D[gc] = \inn{g, c}_D$. Here we assume $\|g\| \leq 1$ when stating lower bounds, again to disallow arbitrary scaling.
	
	Gradient descent with respect to squared loss is captured by inner product queries, since the gradient is given by \begin{align*} \Ex_{x, y} [\nabla_\theta (h_\theta(x) - y)^2] &= \Ex_{x, y} [2(h_\theta(x) - y) \nabla_\theta h_\theta(x)] \\
	&= 2\Ex_x[h_\theta(x) \nabla_\theta h_\theta(x)] \\
	&\quad- 2\Ex_{x,y} [y \nabla_\theta h_\theta(x)]. \end{align*} Here the first term can be estimated directly using knowledge of the distribution, while the latter is a vector each of whose elements is an inner product query.
	
	We now formally define the learning problems we consider.
	
	\begin{definition}[SQ learning of real-valued concepts using inner product queries]
	  Let $\cC$ be a class of $p$-concepts over a domain $X$, and let $D$ be a distribution on $X$. We say that a learner learns $\cC$ with respect to $D$ up to $L_2$ error $\epsilon$ using inner product quiers (equivalently squared loss $\epsilon^2$) if, given only SQ oracle access to $D_c$ for some unknown $c \in \cC$, and using only inner product queries, it is able to output $\tilde{c} : X \to [-1, 1]$ such that $\|c - \tilde{c}\|_D \leq \epsilon$.
	\end{definition}

	For the classification setting, we consider two different notions of learning $p$-concepts. One is learning the target up to small $L_2$ error, to be thought of as a strong form of learning. The other, weaker form, is achieving a nontrivial inner product (i.e.\ unnormalized correlation) with the target. We prove lower bounds on both in order to capture different learning goals.
	
	\begin{definition}[SQ learning of $p$-concepts]\label{def:weak-learning}
		Let $\cC$ be a class of $p$-concepts over a domain $X$, and let $D$ be a distribution on $X$. We say that a learner learns $\cC$ with respect to $D$ up to $L_2$ error $\epsilon$ if, given only SQ oracle access to $D_c$ for some unknown $c \in \cC$, and using arbitrary queries, it is able to output $\tilde{c} : X \to [-1, 1]$ such that $\|c - \tilde{c}\|_D \leq \epsilon$. 
		We say that a learner weakly learns $\cC$ with respect to $D$ with advantage $\epsilon$ if it is able to output $\tilde{c} : X \to [-1, 1]$ such that $\inn{\tilde{c}, c}_D \geq \epsilon$.
	\end{definition}
    Note that the best achievable advantage is $\Ex_{x \sim D}[|c(x)|]$, achieved by $\tilde{c}(x) = \sgn(c(x))$. Note also that $\|c\|_D^2 \leq \Ex_D[|c|] \leq \|c\|_D$, and therefore a norm lower bound on functions in $\cC$ implies an upper bound on the achievable advantage.
	\begin{remark}[Learning with $L_2$ error implies weak learning]
	    If the functions in our class satisfy a norm lower bound, say $\|c\|_D^2 \geq (1 + \alpha) \epsilon^2$, then a simple calculation shows that learning with $L_2$ error $\epsilon$ implies weak learning with advantage $\alpha \epsilon^2/2$.
	    %observe that if $\tilde{c}$ satisfies $\|\tilde{c}-c\|_D \leq \epsilon$, then 
	   % \begin{align*}
	   % \epsilon^2 &\geq \|c - \tilde{c}\|_D^2 \\
	   % &= \|c\|_D^2 + \|\tilde{c}\|_D^2 - 2\inn{c, \tilde{c}}_D \\
	   % &\geq (1 + \alpha)\epsilon^2 - 2\inn{c, \tilde{c}}_D \\
	   % \implies \inn{c, \tilde{c}}_D &\geq \alpha \epsilon^2/2. \end{align*}
	    %\begin{align*}
	    %\inn{c, \tilde{c}}_D &= \frac{\|c\|_D^2 + \|\tilde{c}\|_D^2 - \|c - \tilde{c}\|_D^2}{2} \ge \frac{\alpha \epsilon^2}{2}
	   %\end{align*}
	   % so that $\inn{c, \tilde{c}}_D \geq \alpha \epsilon^2/2$.
	
		Our definition of weak learning also captures the standard boolean sense of weak learning, in which the learner is required to output a boolean hypothesis with 0/1 loss bounded away from $1/2$. Indeed, by an easy calculation, the 0/1 loss of a function $f : X \to \{\pm 1\}$ satisfies \begin{align*} \Pr_{(x, y) \sim D_c}[f(x) \neq y] = \frac{1}{2} - \frac{\inn{c, f}_D}{2}. \end{align*}
		
		\iffalse
		Indeed, in the case that the learner outputs $f : X \to \{\pm 1\}$, observe that the 0/1 loss is given by \begin{align*}
		\underset{{(x, y) \sim D_c}}{\Pr}[f(x) \neq y] &= \Ex_{x \sim D} \left[ \Ind[f(x) \neq 1] \frac{1 + c(x)}{2} + \Ind[f(x) \neq -1] \frac{1 - c(x)}{2} \right] \\
		&= \Ex_{x \sim D} \left[ \left(\frac{1 - f(x)}{2}\right)\left( \frac{1 + c(x)}{2}\right) + \left(\frac{1 + f(x)}{2}\right)\left(\frac{1 - c(x)}{2}\right) \right] \\
		&= \Ex_{x \sim D} \left[ \frac{1}{2} - \frac{f(x)c(x)}{2} \right] \\
		&= \frac{1}{2} - \frac{\inn{c, f}_D}{2}.
		\end{align*}
		\fi
	\end{remark}

	The difficulty of learning a concept class in the SQ model is captured by a parameter known as the statistical dimension of the class. 
	
	\begin{definition}[Statistical dimension] \label{lem:SQ-dim}
		Let $\cC$ be a concept class of either real-valued concepts or $p$-concepts (i.e.\ their corresponding conditional mean functions) on a domain $X$, and let $D$ be a distribution on $X$. The (un-normalized) \emph{correlation} of two concepts $c, c' \in \cC$ under $D$ is $|\inn{c, c'}_D|$.\footnote{In the $p$-concept setting, it is instructive to note that in the notation of \cite{feldman2017statistical}, this correlation is precisely the distributional correlation $\chi_{D_0}(D_c, D_{c'})$ of the induced labeled distributions $D_c$ and $D_{c'}$ under the reference distribution $D_0 = D \times \unif \{\pm 1\}$.} The \emph{average correlation} of $\cC$ is defined to be \begin{align*} \rho_D(\cC) = \frac{1}{|\cC|^2} \sum_{c, c' \in \cC} |\inn{c, c'}_D|. \end{align*}
		The \emph{statistical dimension on average} at threshold $\gamma$, $\sda_D(\cC, \gamma)$, is the largest $d$ such that for all $\cC' \subseteq \cC$ with $|\cC'| \geq |\cC|/d$, $\rho_D(\cC') \leq \gamma$.
	\end{definition}

	\begin{remark}
		For any general and large concept class $\cC^*$ (such as all one-layer neural nets), we may consider a specific subclass $\cC \subseteq \cC^*$ and prove lower bounds on learning $\cC$ in terms of the SDA of $\cC$. These lower bounds extend to $\cC^*$ because if it is hard to learn a subset, then it is hard to learn the whole class.
	\end{remark}
	
	We will mainly be interested in the statistical dimension in a setting where bounds on pairwise correlations are known. In that case the following lemma holds.
	
	\begin{lemma}[adapted from \cite{feldman2017statistical}, Lemma 3.10]\label{lem:sda-bound}
		Suppose a concept class $\cC$ has pairwise correlation $\gamma$, i.e.\ $|\inn{c, c'}_D| \leq \gamma$ for $c \neq c' \in \cC$, and squared norm at most $\beta$, i.e.\ $\|c\|_D^2 \leq \beta$ for all $c \in \cC$. Then for any $\gamma' > 0$, $\sda_D(\cC, \gamma + \gamma') \geq |\cC| \frac{\gamma'}{\beta - \gamma}$. In particular, if $\cC$ is a class of orthogonal concepts (i.e.\ $\gamma = 0$) with squared norm bounded by $\beta$, then $\sda(\cC, \gamma') \geq |\cC| \frac{\gamma'}{\beta}$.
	\end{lemma}
	\begin{proof}
		Let $d = |\cC| \frac{\gamma'}{\beta - \gamma}$, and observe that for any subset $\cC' \subseteq \cC$ satisfying $|\cC'| \geq |\cC|/d = \frac{\beta - \gamma}{\gamma'}$, \begin{align*}
		\rho_{D}(\cC') &= \frac{1}{|\cC'|^2} \sum_{c, c' \in \cC'} |\inn{c, c'}_D| \\
		&\leq \frac{1}{|\cC'|^2} (|\cC'|\beta + (|\cC'|^2 - |\cC'|)\gamma) \\
		&= \gamma + \frac{\beta - \gamma}{|\cC'|} \\
		&= \gamma + \gamma'.
		\end{align*}
	\end{proof}

% 	\paragraph{One-layer neural networks.} 
	\section{Orthogonal Family of Neural Networks}
	We consider neural networks with one hidden layer with activation function $\phi : \R \to \R$, and with one output node that has some activation function $\psi : \R \to \R$. If we take the input dimension to be $n$ and the number of hidden nodes to be $m$, then such a neural network is a function $f : \R^n \to \R$ given by \begin{align*} f(x) = \psi \left( \sum_{i=1}^m a_i \phi(\dotp{w_i}{x}) \right), \end{align*} where $w_i \in \R^n$ are the weights feeding into the $i\th$ hidden node, and $a_i \in \R$ are the weights feeding into the output node.  If $\psi$ takes values in $[-1, 1]$, we may also view $f$ as defining a $p$-concept in terms of its conditional mean function.
	
    For our construction, we need our functions to be orthogonal, and we need a lower bound on their norms. For the first property we only need the distribution on the domain to satisfy a relaxed kind of spherical symmetry that we term sign-symmetry, which says that the distribution must look identical on all orthants. To lower bound the norms, we need to assume that the distribution is Gaussian ${\cal N}(0, I)$. 

	\begin{assumption}[Sign-symmetry] \label{lem:sign-symmetry-assumption}
		For any $z \in \{\pm 1\}^n$ and $x \in \R^n$, let $x \circ z$ denote $(x_1 z_1, \dots, x_n z_n)$. A distribution $D$ on $\R^n$ is \emph{sign-symmetric} if for any $z \in \{\pm 1\}^n$ and $x$ drawn from $D$, $x$ and $x \circ z$ have the same distribution $D$.
	\end{assumption}
	
	\begin{assumption}[Odd outer activation]
	\label{lem:ourer-activation}
	    The outer activation $\psi$ is an odd, increasing function, i.e.\ $\psi(-x) = -\psi(x)$.
	\end{assumption}
	
	Note that $\psi$ could be the identity function. 
	
	\begin{assumption}[Inner activation]
	\label{lem:inner-activation}
	    The inner activation $\phi \in L_2({\cal N}(0, I))$.
	\end{assumption}

	The construction of our orthogonal family of neural networks is simple and exploits sign-symmetry.
	\begin{definition}[Family of Orthogonal Neural Networks]
		Let the domain be $\R^n$, let $\phi : \R \to \R$ be any well-behaved activation function, and let $\psi : \R \to \R$ be any odd function. For an index set $S \subseteq [n]$, let $x_S \in \R^{|S|}$ denote the vector of $x_i$ for $i \in S$. Fix any $k > 0$. For any sign-pattern $z \in \{\pm 1\}^k$, let $\chi(z)$ denote the parity $\prod_i z_i$. For any index set $S \subseteq_k [n]$, define a one-layer neural network with $m = 2^k$ hidden nodes,
		\begin{align*}
		g_S(x) &= \sum_{w \in \{-1, 1\}^k} \chi(w) \phi\left(\frac{\dotp{w}{x_S}}{\sqrt{k}}\right) \\ f_S(x) &= \psi\left(g_S(x)\right).
		\end{align*} 
        
		Our orthogonal family is  
		\[
		\cCo(n, k) = \{ f_S \mid S \subseteq_k [n] \}. 
		\]
		Notice that the size of this family is $\binom{n}{k} = n^{\Theta(k)}$ (for appropriate $k$), which is $n^{\Theta(\log m)}$ in terms of $m$. We will take $k = \Theta(\log n)$, so that $m = \poly(n)$ and thus the neural networks are $\poly(n)$-sized, and the size of the family is $n^{\Theta(\log n)}$, i.e.\ quasipolynomial in $n$.
	\end{definition}
	
	We now prove that our functions are orthogonal under any sign-symmetric distribution.
	\begin{theorem}\label{thm:orth-family}
		Let the domain be $\R^n$, and let $D$ be a sign-symmetric distribution on $\R^n$. Fix any $k > 0$. Then $\inn{f_S, f_T}_D = 0$ for any two distinct $f_S, f_T \in \cCo(n, k)$.
	\end{theorem}
	\begin{proof}
		For the proof, the key property of our construction that we will use is the following: for any sign-pattern $z \in \{\pm 1\}^n$ and any $x \in \R^n$, 
		\begin{equation} \label{eq:parity-property-in-lemma}
		f_S(x \circ z) = \chi_S(z) f_S(x),
		\end{equation} where $\chi_S(z) = \prod_{i \in S} z_i = \chi(z_S)$ is the parity on $S$ of $z$. Indeed, observe first that 
		\begin{align*}
		g_S(x \circ z) &= \sum_{w \in \{-1, 1\}^k} \chi(w) \phi\left(\frac{\dotp{w}{(x \circ z)_S}}{\sqrt{k}}\right) \\
		&= \sum_{w \in \{-1, 1\}^k} \chi(w) \phi\left(\frac{\dotp{(w \circ z_S)}{ x_S}}{\sqrt{k}}\right) \\
		&= \sum_{w \in \{-1, 1\}^k} \chi(w \circ z_S) \chi(z_S) \phi\left(\frac{\dotp{(w \circ z_S)}{x_S}}{\sqrt{k}}\right) \\
		&= \chi(z_S) \sum_{w \in \{-1, 1\}^k} \chi(w)  \phi\left(\frac{\dotp{w}{ x_S}}{\sqrt{k}}\right) \tag*{(replacing $w \circ z_S$ with $w$)} \\
		&= \chi(z_S) g_S(x).
		\end{align*}
		The property then follows since $\psi$ is odd and $\psi(av) = a \psi(v)$ for any $a \in \{\pm 1\}$ and $v \in \R$.
		
		Consider $f_S$ and $f_T$ for any two distinct $S, T \subseteq_k [n]$. Recall that by the definition of sign-symmetry, for any $z \in \{\pm 1\}^n$ and $x$ drawn from $D$, $x$ and $x \circ z$ has the same distribution. Using this and \cref{eq:parity-property-in-lemma}, we have \begin{align*}
		\inn{f_S, f_T}_D &= \Ex_{x \sim D}[f_S(x) f_T(x)] \\
		&= \Ex_{z \sim \{\pm 1\}^n} \ \Ex_{x \sim D}[f_S(x \circ z) f_T(x \circ z)] \tag*{(sign-symmetry)} \\
		&= \Ex_{z \sim \{\pm 1\}^n} \ \Ex_{x \sim D}[\chi_S(z) f_S(x) \chi_T(z) f_T(x)] \tag*{(\cref{eq:parity-property-in-lemma})} \\
		&= \Ex_{x \sim D} \left[ f_S(x) f_T(x) \ \Ex_{z \sim \{\pm 1\}^n} \chi_S(z) \chi_T(z) \right] \\
		&= 0,
		\end{align*} since $\Ex_{z \sim \{\pm 1\}^n} \chi_S(z) \chi_T(z) = 0$ for any two distinct parities $\chi_S, \chi_T$.
	\end{proof}
	
	\begin{remark}
		Our proof actually shows that any family of functions satisfying \cref{eq:parity-property-in-lemma} is an orthogonal family under any sign-symmetric distribution.
	\end{remark}
	
	We still need to establish that our functions are nonzero. For this we need to specialize to the Gaussian distribution, as well as consider specific activation functions (a similar analysis can in principle be carried out for other sign-symmetric distributions). For any $n$ and $k$, it follows from \cref{lem:hermite-norm-expansion} that if the inner activation $\phi$ has a nonzero Hermite coefficient of degree $k$ or higher, then the functions in $\cCo(n, k)$ are nonzero. The sigmoid, ReLU and sign functions all satisfy this property.
	
	\begin{corollary}\label{cor:sda-orth-family}
		Let the domain be $\R^n$, and let $D$ be any sign-symmetric distribution on $\R^n$. For any $\gamma > 0$, 
		\[ \sda_D(\cCo(n, k), \gamma) \geq |\cCo(n,k)| \gamma = \binom{n}{k} \gamma. \]
		Here we also assume that all $c \in \cCo(n,k)$ are nonzero for our distribution $D$.
	\end{corollary}
	\begin{proof}
		Follows from \cref{thm:orth-family} and \cref{lem:sda-bound}, using a loose upper bound of 1 on the squared norm.
	\end{proof}
	
	We also need to prove norm lower bounds on our functions for our notions of learning to be meaningful. In \cref{sec:norm-bounds}, we prove the following.
	
	\begin{theorem}\label{thm:relu-norm-lower-bound}
		Let the inner activation function $\phi$ be $\relu$ or sigmoid, and let the outer activation function $\psi$ be any odd, increasing, continuous function. Let the underlying distribution $D$ be $\mathcal{N}(0, I_n)$. Then $\|f_S\| = \Omega(e^{-\Theta(k)})$, where the hidden constants depend on $\psi$ and $\phi$, for any $f_S \in \cCo(n, k)$.
	\end{theorem}
	
	With this in hand, we now state our main SQ lower bounds.
	
	\begin{theorem}\label{thm:full-learning-bound}
		Let the input dimension be $n$, and let the underlying distribution be $\mathcal{N}(0, I_n)$. Consider $\cCo(n, k)$ instantiated with $\phi = \relu$ or sigmoid and $\psi$ any odd, increasing function (including the identity function), and let $m = 2^k$ be the hidden layer size of each neural net. Let $A$ be an SQ learner using only inner product queries of tolerance $\tau$. For any $k \in \N$, there exists $\tau = 1/n^{-\Theta(k)}$ such that $A$ requires at least $n^{\Omega(k)}$ queries of tolerance $\tau$ to learn $\cCo(n, k)$ with advantage $1/\exp(k)$.
		
		In particular, there exist $k = \Theta(\log n)$ and $\tau = 1/n^{\Theta(\log n)}$ such that $A$ requires at least $n^{\Omega(\log n)}$ queries of tolerance $\tau$ to learn $\cCo(n, k)$ with advantage $1/\poly(n)$. In this case $m = \poly(n)$, so that each function in the family has polynomial size. This is our main superpolynomial lower bound.
	\end{theorem}
	\begin{proof}
		The proof amounts to careful choices of the parameters $\epsilon, \gamma$ and $\tau$ in \cref{cor:sda-orth-family} and \cref{cor:sq-lower-bound}. Recall that $\sda(\cCo(n,k), \gamma) \geq n^{\Theta(k)} \gamma$. We pick $\gamma = n^{-\Theta(k)}$ appropriately such that $d = \sda(\cCo(n,k), \gamma)$ is still $n^{\Theta(k)}$. \cref{thm:relu-norm-lower-bound} gives us a norm lower bound of $\exp(-\Theta(k))$, allowing us to take $\epsilon = \exp(-\Theta(k))$ and $\tau = \sqrt{\gamma} = n^{-\Theta(k)}$ in \cref{cor:sq-lower-bound}.
	\end{proof}

	\section{SQ Lower Bounds} 
	%in terms of statistical dimension}
	
% 	\subsection{SQ lower bounds for real-valued functions}
\paragraph{SQ Lower Bounds for Real-valued Functions}
	
	Prior work \cite{szorenyi2009characterizing, feldman2012complete} has already established the following fundamental result, which we phrase in terms of our definition of statistical dimension. For the reader's convenience, we include a 
% 	self-contained 
	proof in \cref{sec:szorenyi-proof}.
	
	\begin{theorem}
		Let $D$ be a distribution on $X$, and let $\cC$ be a real-valued concept class over a domain $X$ such that $\|c\|_D > \epsilon$ for all $c \in \cC$. Consider any SQ learner that is allowed to make only inner product queries to an SQ oracle for the labeled distribution $D_c$ for some unknown $c \in \cC$. Let $d = \sda_D(\cC, \gamma)$. Then any such SQ learner needs at least $\Omega(d)$ queries of tolerance $\sqrt{\gamma}$ to learn $\cC$ up to $L_2$ error $\epsilon$.
	\end{theorem}
	
	\paragraph{SQ Lower Bounds for {\em p}-concepts}
	
	It turns out to be fruitful to view our learning problem in terms of a decision problem over distributions. We define the problem of distinguishing a valid labeled distribution from a randomly labeled one, and show a lower bound for this problem. We then show that learning is at least as hard as distinguishing, thereby extending the lower bound to learning as well. Our analysis closely follows that of \cite{feldman2017statistical}.
	
	\begin{definition}[Distinguishing between labeled and uniformly random distributions]
		Let $\cC$ be a class of $p$-concepts over a domain $X$, and let $D$ be a distribution on $X$. Let $D_0 = D_{c_0}$ be the randomly labeled distribution $D \times \unif\{\pm 1\}$. Suppose we are given SQ access either to a labeled distribution $D_c$ for some $c \in \cC$ such that $c \neq c_0$ or to $D_0$. The problem of distinguishing between labeled and uniformly random distributions is to decide which.
	\end{definition}
	
    \begin{remark} Given access to $D_c$ for some truly boolean concept $c : X \to \{\pm 1\}$, it is easy to distinguish any other boolean function $c'$ from $c$ since $\|c - c'\|_D^2 = 2 - 2\inn{c, c'}_D$ (which is information-theoretically optimal as a distinguishing criterion) can be computed using a single inner product query. However, if $c$ and $c'$ are $p$-concepts, $\|c\|_D$ and $\|c'\|_D$ are not 1 in general and may be difficult to estimate. It is not obvious how best to distinguish the two, short of directly learning the target. \end{remark}
	
	Considering the distinguishing problem is useful because if we can show that distinguishing itself is hard, then any reasonable notion of learning will be hard as well, including weak learning. We give simple reductions for both our notions of learning.

	\begin{lemma}[Learning is as hard as distinguishing]\label{learning-reduction-combined}
		Let $D$ be a distribution over the domain $X$, and let $\cC$ be a $p$-concept class over $X$. Suppose there exists either 
% \begin{enumerate}[(a)]
% 			\item 

			(a) a weak SQ learner capable of learning $\cC$ up to advantage $\epsilon$ using $q$ queries of tolerance $\tau$, where $\tau \leq \epsilon/2$; or,
% 			\item 

			(b) an SQ learner capable of learning $\cC$ (assume $\|c\|_D \geq 3\epsilon$ for all $c \in \cC$) up to $L_2$ error $\epsilon$ using $q$ queries of tolerance $\tau$, where $\tau \leq \epsilon^2$.
% 		\end{enumerate} 
		Then there exists a distinguisher that is able to distinguish between an unknown $D_c$ and $D_0$ using at most $q+1$ queries of tolerance $\tau$.
	\end{lemma}
	\begin{proof}
% 		\begin{enumerate}[(a)]
% 			\item 
			(a) Run the weak learner to obtain $\tilde{c}$. If $c \neq c_0$, we know that $\inn{\tilde{c}, c}_D \geq \epsilon$, whereas if $c = c_0$, then $\inn{\tilde{c}, c}_D = 0$ no matter what $\tilde{c}$ is. A single additional query ($h(x,y) = \tilde{c}(x)y$) of tolerance $\epsilon/2$ distinguishes between the two cases.
% 			\item 
			
			(b) Run the learner to obtain $\tilde{c}$. If $c \neq c_0$, i.e.\ $\|c\|_D \geq 3\epsilon$, we know that $\|\tilde{c} - c\|_D \leq \epsilon$, so that by the triangle inequality, $ \|\tilde{c}\|_D \geq \|c\|_D - \|\tilde{c} - c\|_D \geq 2\epsilon$. But if $c = c_0$, then $ \|\tilde{c}\|_D \leq \epsilon$. An additional query ($h(x,y) = \tilde{c}(x)^2$) of tolerance $\epsilon^2$ suffices to distinguish the two cases.
% 		\end{enumerate}
	\end{proof}
	
	We now prove the main lower bound on distinguishing.
	
	\begin{theorem}\label{distinguishing-lower-bound}
		Let $D$ be a distribution over the domain $X$, and let $\cC$ be a $p$-concept class over $X$. Then any SQ algorithm needs at least $d = \sda(\cC, \gamma)$ queries of tolerance $\sqrt{\gamma}$ to distinguish between $D_c$ and $D_0$ for an unknown $c \in \cC$. (We will consider deterministic SQ algorithms that always succeed, for simplicity.)
	\end{theorem}
	\begin{proof}
		Consider any successful SQ algorithm $A$. Consider the adversarial strategy where to every query $h : X \times \{-1, 1\} \to [-1, 1]$ of $A$ (with tolerance $\tau = \sqrt{\gamma}$), we respond with $\Ex_{D_0}[h]$. We can pretend that this is a valid answer with respect to any $c \in \cC$ such that $|\Ex_{D_c}[h] - \Ex_{D_0}[h]| \leq \tau$. Our argument will be based on showing that each such query rules out fairly few distributions, so that the number of queries required in total is large.
		
		Since we assumed that $A$ is a deterministic algorithm that always succeeds, it eventually correctly guesses that it is $D_0$ that it is getting answers from. Say it takes $q$ queries to do so. For the $k\th$ query $h_k$, let $S_k$ be the set of concepts in $\cC$ that are ruled out by our response $\Ex_{D_0}[h_k]$: \[ S_k = \{ c \in \cC \mid |\Ex_{D_c}[h] - \Ex_{D_0}[h]| \ > \tau \}. \]
		We'll show that
% 		\begin{enumerate}[(a)]
% 			\item 
			
			(a) on the one hand, $\cup_{k=1}^q S_k = \cC$, so that $\sum_{k=1}^q |S_k| \geq |\cC|$,
% 			\item 
			
			(b) while on the other, $|S_k| \leq |\cC|/d$ for every $k$.
% 		\end{enumerate} 
		Together, this will mean that $q \geq d$.
		
		For the first claim, suppose $\cup_{k=1}^q S_k$ were not all of $\cC$, and indeed say $c \in \cC \setminus (\cup_{k=1}^q S_k)$. This is a distribution that our answers were consistent with throughout, yet one that $A$'s solution ($D_0$) is incorrect for. But $A$ always succeeds, so for it not to have ruled out this $D_c$ is impossible.
		
		For the second claim, suppose for the sake of contradiction that for some $k$, $|S_k| > |\cC|/d$. By \cref{lem:SQ-dim}, this means we know that $\rho_D(S_k) \leq \gamma$. One of the key insights in the proof of \cite{szorenyi2009characterizing} is that by expressing query expectations entirely in terms of inner products, we gain the ability to apply simple algebraic techniques. To this end, for any query function $h$, let $\hat{h}(x) = (h(x, 1) - h(x, -1))/2$. Observe that for any $p$-concept $c$, \begin{align*} \inn{\hat{h}, c}_D &= \Ex_{x \sim D}\left[h(x, 1)\frac{c(x)}{2}\right] - \Ex_{x \sim D}\left[h(x, -1)\frac{c(x)}{2}\right] \\
		&= \Ex_{x \sim D}\left[h(x, 1)\frac{1+c(x)}{2}\right] \\ &\quad + \Ex_{x \sim D}\left[h(x, -1)\frac{1-c(x)}{2}\right] 
		\\ &\quad  - \Ex_{x \sim D}\left[h(x, 1)\frac{1}{2}\right]
		- \Ex_{x \sim D}\left[h(x, -1)\frac{1}{2}\right] \\
		&= \Ex_{D_c}[h] - \Ex_{D_0}[h], \end{align*} the difference between the query expectations wrt $D_c$ and $D_0$. Here we have expanded each $\Ex_{D_c}[h]$ using the fact that the label for $x$ is $1$ with probability $(1 + c(x))/2$ and $-1$ otherwise. Thus $|\inn{\hat{h_k}, c}_D|$, where $h_k$ is the $k\th$ query, is greater than $\tau$ for any $c \in S_k$, since $S_k$ are precisely those concepts ruled out by our response. We will show contradictory upper and lower bounds on the following quantity: \begin{align*} \Phi = \inn*{\hat{h_k}, \sum_{c \in S_k} c \cdot \sgn(\inn{\hat{h_k}, c}_D)}_{D}. \end{align*} Note that since every query $h$ satisfies $\|h(\cdot, y)\|_D \leq 1$ for all $y$, it follows by the triangle inequality that $\|\hat{h}\|_D \leq 1$.	So by Cauchy-Schwarz and our observation that $\rho_D(S_k) \leq \gamma$, \begin{align*}
		\Phi^2 &\leq  \|\hat{h_k}\|_D^2 \cdot \left\|\sum_{c \in S_k} c \cdot \sgn(\inn{\hat{h_k}, c})\right\|_D^2 \\
		&\leq  \sum_{c, c' \in S_k} |\inn{c, c'}_D| = |S_k|^2 \rho_{D}(S_k) \leq |S_k|^2 \gamma.
		\end{align*}
		
		However since $|\inn{\hat{h_k}, c}_D| \ > \tau$, we also have that $\Phi = \sum_{c \in S_k} |\inn{\hat{h_k}, c}_D| \ > |S_k| \tau.$ 
% 		\begin{align*}
% 		\Phi = \sum_{c \in S_k} |\inn{\hat{h_k}, c}| \ > |S_k| \tau.
% 		\end{align*} 
		Since $\tau = \sqrt{\gamma}$, this contradicts our upper bound and in turn completes the proof of our second claim. And as noted earlier, the two claims together imply that $q \geq d$.
	\end{proof}
	
	%\begin{remark}
		%The above bound can be extended to randomized learners using a Yao's principle argument, but we omit the slightly tedious details in the interest of simplicity.
	%\end{remark}
	
	The final lower bounds on learning thus obtained are stated as a corollary for convenience. The proof follows directly from Lemma \ref{learning-reduction-combined} and Theorem \ref{distinguishing-lower-bound}.
	
	\begin{corollary}\label{cor:sq-lower-bound}
		Let $D$ be a distribution over the domain $X$, and let $\cC$ be a $p$-concept class over $X$. Let $\gamma, \tau$ be such that $\sqrt{\gamma} \leq \tau$. Let $d = \sda(\cC, \gamma)$.
% 		\begin{enumerate}[a]
% 			\item 
			
			(a) Let $\epsilon$ be such that $\tau \leq \epsilon^2$, and assume $\|c\|_D \geq 3\epsilon$ for all $c \in \cC$. Then any SQ learner learning $\cC$ up to $L_2$ error $\epsilon$ requires at least $d - 1$ queries of tolerance $\tau$.
% 			\item 
			
			(b) Let $\epsilon$ be such that $\tau \leq \epsilon/2$. Then any weak SQ learner learning $\cC$ up to advantage $\epsilon$ requires at least $d - 1$ queries of tolerance $\tau$.
% 		\end{enumerate}
	\end{corollary}

\section{Experiments}
We include experiments for both regression and classification. We train an overparameterized neural network on data from our function class, using gradient descent. We find that we are able to achieve close to zero training error, while test error remains high. This is consistent with our lower bound for these classes of functions.
\begin{figure}[p]
\centering
    \begin{subfigure}[t]{0.48\textwidth}
        \centering
    \includegraphics[width=\textwidth]{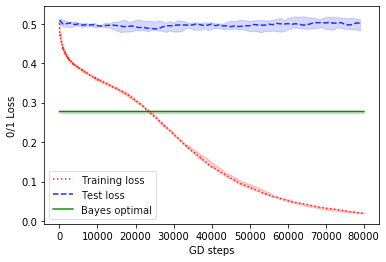}
    \caption{Learning a softmax of a one-layer tanh network}
    \end{subfigure}
    ~~
    \begin{subfigure}[t]{0.48\textwidth}
        \centering
    \includegraphics[width=\textwidth]{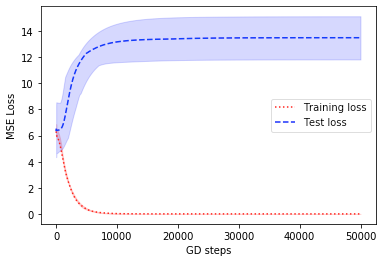}
    \caption{Learning a linear combination of tanhs}
    \end{subfigure}
    \caption{In (a) the target function is a softmax ($\pm 1$ labels) of a sum of $2^7$ tanh activations with $n = 14$; in (b) the labels are obtained similarly but without the softmax.
    In both cases, we train a 1-layer neural network with $5 \cdot 2^7 = 640$ tanh units (hence $10241$ parameters) using a training set of size $6000$ and a test set of size $1000$, with the learning rate set to $0.01$. For (a) we take the sign of this trained network and measure its training and testing 0/1 loss; for (b) we measure the train and test square-loss of the learned network directly. In (a) we also plot the test error of the bayes optimal network (sign of the target function).
    }
    \label{fig:softmax_tanh}
\end{figure}	

For classification, we use a training set of size $T$ of data corresponding to $f \in \cCo(n, k)$ instantiated with $\phi = \tanh$ and $\psi = \tanh$. We draw $x \sim \mathcal{N}(0, I_n)$. For each $x$, $y$ is picked randomly from $\{\pm 1\}$ in such a way that $\Ex[y | x] = f(x)$. Since the outer activation $\psi$ is $\tanh$, this can be thought of as applying a softmax to the network's output, or as the Boolean label corresponding to a logit output. We train a sum of tanh network (i.e.\ a network in which the inner activation is $\tanh$ and no outer activation is applied) on this data using gradient descent on squared loss, threshold the output, and plot the resulting 0/1 loss. See \cref{fig:softmax_tanh}(a). This setup models a common way in which neural networks are trained for classification problems in practice.

For regression, we use a training set of size $T$ of data corresponding to $f \in \cCo(n, k)$ instantiated with $\phi = \tanh$ and $\psi$ being the identity. We draw $x \sim \mathcal{N}(0, I_n)$, and $y = f(x)$. We train a sum of tanh network on this data using gradient descent on squared loss, which we plot in \cref{fig:softmax_tanh}(b). This setup models the natural way of using neural networks for regression problems.

In both cases, we train neural networks whose number of parameters considerably exceeds the amount of training data. In all our experiments, we plot the median over 10 trials and shade the inter-quartile range of the data. 

Similar results hold with the inner activation $\phi$ being $\relu$ instead of $\tanh$, and are shown in \cref{fig:softmax_relu-2}.

\begin{figure}[p]
\centering
    \begin{subfigure}{0.48\textwidth}
        \centering
    \includegraphics[width=\textwidth]{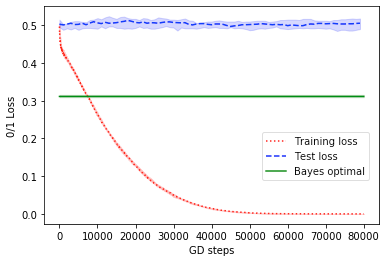}
    \caption{Learning a softmax of a one-layer ReLU network}
    \end{subfigure}
    ~~
    \begin{subfigure}{0.48\textwidth}
        \centering
    \includegraphics[width=\textwidth]{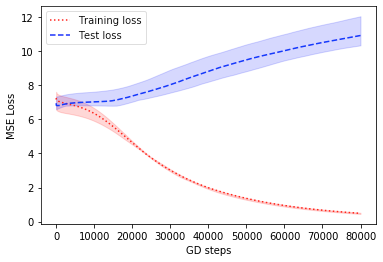}
    \caption{Learning a linear combination of ReLUs}
    \end{subfigure}
    \caption{In (a) the target function is a softmax ($\pm 1$ labels) of a sum of $2^8$ ReLU activations with $n = 14$; in (b) the labels are obtained similarly but without the softmax.
    In both cases, we train a 1-layer neural network with $5 \cdot 2^8 = 1280$ ReLU units (hence $20481$ parameters) using a training set of size $6000$ and a test set of size $1000$, with the learning rate set to $0.005$ for classification and $0.002$ for regression. For (a) we take the sign of this trained network and measure its training and testing 0/1 loss; for (b) we measure the train and test square loss of the learned network directly. In (a) we also plot the test error of the bayes optimal network (sign of the target function).
    }
    \label{fig:softmax_relu-2}
\end{figure}

\clearpage
\bibliography{refs}
\bibliographystyle{alpha}
	
\appendix
\section{Bounding the function norms under the Gaussian}\label{sec:norm-bounds}

Our goal in this section will be to give lower bounds on the norms of the functions in $\cCo(n,k)$, which is a technical requirement for our results to hold (see \cref{learning-reduction-combined} and \cref{cor:sq-lower-bound}). Note that when learning with respect to $L_2$ error, such a lower bound is necessary if we wish to state SQ lower bounds, since if the target had small norm, say $\|f\|_D \leq \epsilon$, then the zero function trivially achieves $L_2$ error $\epsilon$.

All inner products and norms in this section will be with respect to the standard Gaussian, $\mathcal{N}(0, I)$. 
Since we will fix $S$ throughout, for our purposes the only relevant part of the input is $x_S$ and so we drop the subscripts and let $g = g_S, f = f_S$ and $x = x_S$, so that $g$ and $f$ are functions of $x \in \R^k$. 
Our approach will be as follows. 
%For convenience, let $f_S(x) = \psi(g_S(x))$, so that $g_S$ is $f_S$ before the final activation. 
In order to prove a norm lower bound on $f$, we will prove an anticoncentration result for $g$. To this end we first calculate the second moment of $g$ in terms of the Hermite coefficients of $\phi$. 

\begin{lemma} \label{lem:hermite-norm-expansion}
	Under the distribution $\mathcal{N}(0, I_n)$, let the Hermite representation of $\phi$ be $\phi(x) = \sum_{i = 0}^\infty \hat{\phi_i} \nher_i(x)$, where $\nher_i(x)$ is the $i\th$ normalized probabilists' Hermite polynomial. Then
	\begin{align*}
		\Ex\left[g(x)^2\right] 
		= 4^k \sum_{i \ge 0} \frac{\hat{\phi_i}^2}{k^i} \sum_{\substack{i_1 + \cdots + i_k = i \\ i_1, \dots, i_k\text{ are odd}}} \binom{i}{i_1, \dots, i_k}.
	\end{align*}
\end{lemma}

\begin{proof}
	We use $\Ex$ in this proof instead of $\Ex_{x \sim \mathcal{N}(0, I_n)}$ for simplicity. Then we have 
	\begin{align*}%\nolabel
	&\Ex\!\left[g(x)^2\right] \\
	% 		=&\, \Ex\left[\sum_{w \in \{\pm 1\}^k} \chi(w) \phi\left(\frac{\dotp{w}{x_S}}{\sqrt{k}}\right) \right]^2 \\
	=&\, \Ex\!\left[\sum_{\alpha \in \{\pm 1\}^k} \chi(\alpha) \phi\left(\frac{\dotp{\alpha}{ x_S}}{\sqrt{k}}\right) \right] \!\!\!\left[\sum_{\beta \in \{\pm 1\}^k} \chi(\beta) \phi\left(\frac{\dotp{\beta}{x_S}}{\sqrt{k}}\right) \right] \\
	=&\, \sum_{\alpha,\beta \in \{\pm 1\}^k} \prod_{l = 1}^k \alpha_l \beta_l\, \Ex\!\left[\phi\left(\frac{\dotp{\alpha}{x_S}}{\sqrt{k}}\right) \phi\left(\frac{\dotp{\beta}{x_S}}{\sqrt{k}}\right) \right] \\
	=&\, \sum_{\alpha,\beta \in \{\pm 1\}^k} \prod_{l = 1}^k \alpha_l \beta_l\, \Ex\!\!\left[\sum_{i, j \ge 0} \hat{\phi_i} \hat{\phi_j}\nher_i\!\left(\frac{\dotp{\alpha}{x_S}}{\sqrt{k}} \right)\! \nher_j\!\left(\frac{\dotp{\beta}{x_S}}{\sqrt{k}} \right)\! \right] \\
	=&\, \sum_{\alpha,\beta \in \{\pm 1\}^k} \prod_{l = 1}^k \alpha_l \beta_l \!\sum_{i, j \ge 0} \hat{\phi_i} \hat{\phi_j}\, \Ex\!\left[\nher_i\!\left(\frac{\dotp{\alpha}{x_S}}{\sqrt{k}} \right) \! \nher_j\!\left(\frac{\dotp{\beta}{x_S}}{\sqrt{k}} \right) \right].
	\end{align*}
	Since $x \sim \mathcal{N}(0, I_k)$, $\frac{\inn{\alpha, x_S}}{\sqrt{k}}$ and $\frac{\inn{\beta, x_S}}{\sqrt{k}}$ are both standard Gaussian and have correlation $\frac{\inn{\alpha, \beta}}{k}$, we then apply the following well-known property of the Hermite polynomials.
	\[
		\Ex_{(a, b)^T \sim \mathcal{N}\left(0, \bigl( \begin{smallmatrix} 1 & \rho \\ \rho & 1 \end{smallmatrix} \bigr)\right)} \nher_i(a) \nher_j(b) = \delta_{i, j} \rho^i,
	\]
	where $\delta_{i, j}$ is the Dirac delta function.
	\begin{align*}
	\Ex\left[g(x)^2\right]
	% 		=&\, \sum_{\alpha,\beta \in \{\pm 1\}^k} \prod_{l = 1}^k \alpha_l \beta_l \sum_{i, j \ge 0} \hat{\phi_i} \hat{\phi_j}\, \Ex\!\left[\nher_i\left(\frac{\dotp{\alpha}{x_S}}{\sqrt{k}} \right) \nher_j\left(\frac{\dotp{\beta}{x_S}}{\sqrt{k}} \right) \right] \\
	=&\, \sum_{\alpha,\beta \in \{\pm 1\}^k} \prod_{l = 1}^k \alpha_l \beta_l \sum_{i \ge 0} \hat{\phi_i}^2 \left(\frac{\dotp{\alpha}{\beta}}{k}\right)^i\\
	=&\, \sum_{w, \theta \in \{\pm 1\}^k} \prod_{l = 1}^k w_l \sum_{i \ge 0} \hat{\phi_i}^2 \left(\frac{\sum_{l = 1}^k w_l}{k} \right)^i \\
	=&\, 2^k \sum_{w \in \{\pm 1\}^k} \prod_{l = 1}^k w_l \sum_{i \ge 0} \hat{\phi_i}^2 \left(\frac{\sum_{l = 1}^k w_l}{k} \right)^i,
	\end{align*}
	where $w_i = \alpha_i \beta_i$ and $\theta_i = w_i \alpha_i$. Note that \cref{lem:inner-activation} implies that $\sum_{i = 0}^\infty \hat{\phi_i}^2 < \infty$  , the series above is absolute convergent. Then, 
	\begin{align*}
	&\, \Ex\left[g(x)^2 \right]  \\
	=&\, 2^k \sum_{i \ge 0} \hat{\phi_i}^2 \sum_{w \in \{\pm 1\}^k} \prod_{l = 1}^k w_l \left(\frac{\sum_{l = 1}^k w_l}{k} \right)^i \\
	=&\, 2^k \sum_{i \ge 0} \frac{\hat{\phi_i}^2}{k^i} \sum_{w \in \{\pm 1\}^k} \prod_{l = 1}^k w_l \sum_{i_1 + \cdots + i_k = i} \prod_{l = 1}^k w_{l}^{i_l} \binom{i}{i_1, \dots, i_k} \\
	=&\, 2^k \sum_{i \ge 0} \frac{\hat{\phi_i}^2}{k^i} \sum_{i_1 + \cdots + i_k = i} \binom{i}{i_1, \dots. i_k} \sum_{w \in \{\pm 1\}^k} \prod_{l = 1}^k w_{l}^{i_l + 1} \\
	=&\, 2^k \sum_{i \ge 0} \frac{\hat{\phi_i}^2}{k^i} \sum_{i_1 + \cdots + i_k = i} \binom{i}{i_1, \dots, i_k} \prod_{l = 1}^k \left[1^{i_l + 1} + (-1)^{i_l + 1}\right] \\
	=&\, 4^k \sum_{i \ge 0} \frac{\hat{\phi_i}^2}{k^i} \sum_{\substack{i_1 + \cdots + i_k = i \\ i_1, \dots, i_k\text{ are odd}}} \binom{i}{i_1, \dots, i_k}
	\end{align*}
	since we consider all distinct monomials in $\big(\sum_{l = 1}^k w_l\big)^i$. Note that $\sum_{\substack{i_1 + \cdots + i_k = i \\ i_1, \dots, i_k\text{ are odd}}} \binom{i}{i_1, \dots, i_k}$ is always non-negative and is positive iff $i \ge k$ and $i \equiv k \pmod{2}$.
	% 		Also, by enumerating over $\sum_{l = 1}^k w_l$, we obtain
	% 		\[
	% 		\Ex\left[\sum_{w \in \{\pm 1\}^k} \chi(w) \phi\left(\frac{\inn{w, x_S}}{\sqrt{k}}\right) \right]^2 = 2^k \sum_{i \ge k} \frac{\hat{\phi_i}^2}{k^i} \sum_{l = 0}^k \binom{k}{l} (-1)^l (k - 2l)^i.
	% 		\]
\end{proof}

\subsection{ReLU Activation}
The goal of this section is to give a lower-bound of $\norm{f}$ for $\phi = \relu$ under the standard Gaussian distribution $\mathcal{N}(0, I)$. To this end, we prove an anti-concentration for $g$.
% Now, our approach is as follows. Letting $f = \psi \circ g$, so that $g$ denotes the inner part of the neural network, proving a norm lower bound on $f$ requires proving an anticoncentration inequality for $g$. 
We first give a lower bound on $\norm{g}$ based on the Hermite coefficients of $\phi$. If $g$ were bounded, this alone would imply anti-concentration as in \cref{sigmoid-section}. But since it is not, we first introduce $g^T$, where all activations are truncated at some $T$. We pick $T$ large enough that $g$ and $g^T$ behave almost identically over ${\cal N}(0, I)$. We then show a lower bound on $\norm{g^T}$, translate that into an anticoncentration result for $g^T$, and finally into one for $g$.

Let $T > 0$ be some constant to be determined later. Let 
\begin{equation} \nonumber
  \relu^T(x) = \min(\relu(x), T)
\end{equation}
and
% \begin{equation} \nonumber
% \begin{cases}
% g(x) = \sum\limits_{w \in \{\pm 1\}^k} \chi(w) \relu(\inn{x, w}) \\
% g^T(x) = \sum\limits_{w \in \{\pm 1\}^k} \chi(w) \relu^T(\inn{x, w}). \\
% \end{cases}
% \end{equation}
\[
  g^T(x) = \sum\limits_{w \in \{\pm 1\}^k} \chi(w) \relu^T\left(\frac{\dotp{x}{w}}{\sqrt{k}}\right).
\]
The following lemma from \cite{goel2019time} describes the Hermite coefficients of ReLU. 
\begin{lemma}
	\begin{align*}
		\relu(x) = \sum_{i = 0}^\infty c_i \nher_i(x)
	\end{align*}
	where 
	\begin{align*}
		c_0 = \sqrt{\frac{1}{2 \pi}}, \quad&
		c_1 = \frac{1}{2}, \\
		c_{2i-1} = 0, \quad&
		c_{2i} = \frac{H_{2i}(0) + 2i H_{2i-2}(0)}{\sqrt{2\pi (2i)!}} \quad \text{for } i \ge 2.
	\end{align*}
	In particular, $c_{2i}^2 = \Theta(i^{-2.5})$.
	%\Theta(\sqrt{k}, i^{-\frac{5}{4}})$.
\end{lemma}

We can now derive a lower bound on the norm of $g$.
\begin{lemma} \label{lem:g-relu-norm}
  When $k$ is even, 
  \begin{align*}
	\norm{g} = \Omega\left( \left(\frac{4}{e}\right)^
	{(\frac{1}{2} + o(1))k }\right).
  \end{align*}
\end{lemma}
\begin{proof}
	Due to \cref{lem:hermite-norm-expansion},
	\begin{align*}
	\Ex\left[g(x)^2 \right]
	&= 4^k \sum_{i \ge 0} \frac{c_i^2}{k^i} \sum_{\substack{i_1 + \cdots + i_k = i \\ i_1, \dots, i_k, \text{ are odd}}} \binom{i}{i_1, \dots, i_k} \\
	&\ge \frac{4^k c_k^2}{k^k} \sum_{\substack{i_1 + \cdots + i_k = k \\ i_1, \dots, i_k, \text{ are odd}}} \binom{k}{i_1, \dots, i_k} \\
	&\ge \frac{4^k c_k^2 k!}{k^k}.
	\end{align*}
	The lemma then follows by the Stirling's approximation, 
	\begin{align*}
		n! \ge \sqrt{2\pi n} \left(\frac{n}{e}\right)^n.
	\end{align*}
	and the bound on the Hermite coefficients,
	\begin{align*} c_k^2 = \Theta(k^{-2.5}). \end{align*}
\end{proof}
For the difference of $g(x)$ and $g^T(x)$, we have
\begin{lemma} \label{lem:diff-norm}
	\begin{align*}
		\norm{g - g^T} \le 2^k\, e^{-\frac{T^2}{4}} \sqrt{T^2 + 1 - \frac{T}{\sqrt{2\pi}}}
	\end{align*}
\end{lemma}
\begin{proof}

    Let $\relu_w(x)$ be shorthand for $\relu(\frac{\dotp{x}{w}}{\sqrt{k}})$, and similarly $\relu_w^T$. Observe that by the triangle inequality, \begin{align*}
        \norm{g - g^T} &= \norm{\sum_{w \in \{\pm 1\}^k} \chi(w) \left(\relu_w - \relu_w^T \right) } \\
        &\leq \sum_{w \in \{\pm 1\}^k} \norm{\relu_w - \relu_w^T} \\
        &= 2^k \norm{\relu - \relu^T}_{\mathcal{N}(0, 1)},
    \end{align*} where the last equality holds because for any unit vector $v$ and $x \sim \mathcal{N}(0, I)$, $\dotp{x}{v}$ has the distribution $\mathcal{N}(0, 1)$. Now, \[ \norm{\relu - \relu^T}_{\mathcal{N}(0, 1)}^2 = \int_{T}^\infty (x - T)^2\, p(x)\, dx, \]
	where $p(x)$ is the probability density function of $\mathcal{N}(0, 1)$. Note that $p'(x) = -xp(x)$. We have
	\begin{align*}
	& \int_{T}^\infty x^2 p(x) dx  = \int_{T}^\infty -x\, d(p(x)) \\
	&\qquad = -x\,p(x)\bigg|_{T}^\infty + \int_{T}^\infty p(x) dx \tag*{\text{(integration by parts)}} \\
	&\qquad = T\, p(T) + \Pr_{x \sim_{\mathcal{N}(0, 1)}}(x > T), \\
	& \int_{T}^\infty x\, p(x) dx = -p(x)\bigg|_{T}^\infty = p(T), \\
	& \int_{T}^\infty p(x) dx = \Pr_{x \sim_{\mathcal{N}(0, 1)}}(x > T) \le e^{-\frac{T^2}{2}}.
	\end{align*} Thus, 
	\begin{align*}
	&\Ex\left[g(x) - g^T(x) \right]^2 \\
	&\le 4^k\, \left[(T^2 + 1) \Pr_{x \sim \mathcal{N}(0, 1)}(x > T) - T\, p(T) \right] \\
	&\le 4^k\, e^{-\frac{T^2}{2}} \left( T^2 + 1 - \frac{T}{\sqrt{2\pi}} \right).
	\end{align*}
\end{proof}

\begin{lemma} \label{lem:diff-prob}
	\begin{align*}
		\Pr[g(x) \neq g^T(x)] \le 2^k\, e^{-\frac{T^2}{2}}.
	\end{align*}
\end{lemma}
\begin{proof}
	For any $w \in \{\pm 1\}^k$,
	\begin{align*}
		&\Pr_{x \sim \mathcal{N}(0, I)} \left[ \relu(\frac{\dotp{x}{w}}{\sqrt{k}}) \neq \relu^T(\frac{\dotp{x}{w}}{\sqrt{k}}) \right] \\
		&= \Pr_{t \sim \mathcal{N}(0, 1)} [t > T] \\
		&\le e^{-\frac{T^2}{2}}.
	\end{align*} The lemma follows by a union bound.
\end{proof}

\begin{lemma}\label{lem:g-anticoncentration}
	\begin{align*}
		\Pr\left[\abs{g(x)} \ge 1\right] = \Omega(\exp(-\Theta(k))).
	\end{align*}
\end{lemma}
\begin{proof}
	For large enough $T = \Omega(k)$, it holds from \cref{lem:g-relu-norm,lem:diff-norm} that
	\begin{align*}
		\norm{g^T} = \Omega\left( \left(\frac{4}{e}\right)^
	{(\frac{1}{2} + o(1))k }\right).
	\end{align*}
	Since $\abs{g^T(x)} \le T \, 2^k$,
	\[ \norm{g^T}^2 = \Ex[g^T(x)^2] \leq 1 \cdot \Pr[|g^T(x)| \leq 1] + (T 2^k)^2 \cdot \Pr[|g^T(x)| \geq 1], \] so that
	\begin{equation} \label{eq:anti-gT-1}
	\Pr\left[\abs{g^T(x)} \ge 1\right] = \dfrac{\Omega\Big( \left(\frac{4}{e}\right)^
	{(1 + o(1))k }\Big) - 1}{(T \, 2^k)^2} =  \Omega(\exp(-\Theta(k)))
	\end{equation}
	Using \cref{eq:anti-gT-1} with \cref{lem:diff-prob}, 
	\begin{align*}
		\Pr\left[\abs{g(x)} \ge 1\right] = \Omega(\exp(-\Theta(k)))
	\end{align*}
	for large enough $T = \Omega(k)$.
\end{proof}

The lower bound on $\norm{f}$ now follows easily.
\begin{corollary} \label{lem:f-norm-relu}
	\begin{align*} \|f\| = \Omega(\exp(-\Theta(k))). \end{align*}
\end{corollary}
\begin{proof}
	Since $f = \psi \circ g$, from \cref{lem:g-anticoncentration} and the fact that $\psi$ is odd and increasing, we have that \begin{align*} \|f\| &\geq |\psi(1)| \ \Pr[g(x) \geq 1] + |\psi(-1)| \ \Pr[g(x) \geq 1] \\
	&= \psi(1) \ \Pr[|g(x)| \geq 1] \\
	&= \Omega(\exp(-\Theta(k))). \end{align*}
\end{proof}

\subsection{Sigmoid Activation} \label{sigmoid-section}
Here we consider $g$ and $f$ with $\phi(x) = \sigma(x) = \frac{1}{1 + e^{-x}}$. For the asymptotic bound of Hermite polynomial coefficients, we need the following theorem from~\cite{boyd1984asymptotic}.
\begin{theorem}
	For a function $f(z)$ whose convergence is limited by simple poles at the roots of $z^2 = -\gamma^2$ with residue $R$, the non-zero expansion coefficients $\{a_n\}$ of $f(z)$ as a series of normalized Hermite functions have magnitudes asymptotically given by 
	\[
		\abs{a_n} \sim 2^{\frac{5}{4}}\, \pi^{\frac{1}{2}}\, R\, n^{-\frac{1}{4}}\, e^{-\gamma (2n+1)^{\frac{1}{2}}},
	\]
	Here the normalized Hermite function $\{\psi_n(x)\}_{n \in \N}$ is defined by
	\[
	  \psi_n(z) = e^{-\frac{z^2}{2}} \pi^{-\frac{1}{4}} \tilde{H}_{n}(\sqrt{2} z).
	\]
\end{theorem}
Applying this to $f(x) = e^{-\frac{x^2}{2}} \sigmoid(\sqrt{2}x)$ and translating the Hermite coefficients for the series in terms of Hermite functions to those in terms of Hermite polynomials, we have 
\begin{lemma}
	\begin{equation*}
	\sigmoid(x) = \sum_{i=0}^\infty c_i \tilde{H}_i(x),
	\end{equation*}
	where $c_0 = 0.5, c_{2i} = 0$ for $i \ge 1$ and all non-zero odd terms satisfies
	\begin{equation*}
	c_{2i-1} = e^{-\Theta(\sqrt{i})}.
	\end{equation*}
\end{lemma}

\begin{corollary} \label{lem:sigmoid-coefficient}
	There is an infinite increasing sequence $\{k_i\}_{i \in \N}$ such that $k_i$'s are all odd and 
	\[
		c_{k_i} = e^{-\Theta(\sqrt{k_i})}.
	\]
\end{corollary}
\begin{proof}
	It follows simply from the fact that $\sigmoid$ is not a polynomial and there should be infinitely many non-zero terms in $\{c_k\}_{k \in \N}$.
\end{proof}

\begin{remark}
	Experimental evidence strongly indicates that in fact all odd Hermite coefficients of sigmoid are nonzero and decay as above, but this is laborious to formally establish. So we state our norm lower bound only for $k \in \{k_i\}_{i \in \N}$ (and the associated $n \in \{2^{k_i}\}_{i \in \N}$, since we end up taking $k = \log n$). Since this is nevertheless an infinite sequence, it still establishes that no better asymptotic bound holds.
\end{remark}

Similar to \cref{lem:g-relu-norm}, we can derive a lower bound of $\norm{g}$ for some $k$'s.
\begin{lemma}\label{lem:g-sigmoid-norm}
	For $k \in \{k_i\}_{i \in \N}$,
	\[
		\norm{g(x)} = \Omega\left( \left(\frac{4}{e}\right)^{(\frac{1}{2}+o(1))k}\right).
	\]
\end{lemma}
\begin{proof}
	Due to \cref{lem:hermite-norm-expansion},
	\begin{align*}
	\Ex\left[g(x)^2 \right]
	&= 4^k \sum_{i \ge 0} \frac{c_i^2}{k^i} \sum_{\substack{i_1 + \cdots + i_k = i \\ i_1, \cdots, i_k, \text{ are odd}}} \binom{i}{i_1 \cdots i_k} \\
	&\ge \frac{4^k c_k^2}{k^k} \sum_{\substack{i_1 + \cdots + i_k = k \\ i_1, \cdots, i_k, \text{ are odd}}} \binom{k}{i_1 \cdots i_k} \\
	&\ge \frac{4^k c_k^2 k!}{k^k}.
	\end{align*}
	Using Stirling's approximation, 
	\[
		n! \ge \sqrt{2\pi n} \left(\frac{n}{e}\right)^n,
	\]
	and \cref{lem:sigmoid-coefficient},
	\[
		c_k = e^{-\Theta(\sqrt{k})},
	\]
	we obtain 
	\[
		\Ex\left[g(x)^2 \right] = \Omega\left( \frac{4^k \sqrt{2 \pi k}}{k^k} \left(\frac{k}{e}\right)^k e^{-\Theta(\sqrt{k})}\right)
	\]
	and hence
	\[
		\Ex\left[g(x)^2 \right] = \Omega\left( \left(\frac{4}{e}\right)^{(1+o(1))k}\right).
	\]
\end{proof}
\begin{lemma}
	For $k \in \{k_i\}_{i \in \N}$,
	\[
		\Pr\left(\abs{g(x)} \ge 1\right) = \Omega(\exp(-\Theta(k))).
	\]
\end{lemma}
\begin{proof}
	Since $\abs{g(x)} \le 2^k$,
	\[ \norm{g}^2 = \Ex[g(x)^2] \leq 1 \cdot \Pr[|g(x)| \leq 1] + (2^k)^2 \cdot \Pr[|g(x)| \geq 1], \] and so
	\[
	\Pr\left(\abs{g(x)} \ge 1\right) = \dfrac{\Omega\Big( \big(\frac{4}{e}\big)^{(1+o(1))k} \Big) - 1}{(2^k)^2}.
	\]
	The lemma then follows.
\end{proof}
Using the same argument as \cref{lem:f-norm-relu}, we have the following bound.
\begin{corollary} \label{lem:f-norm-sigmoid}
	\[ \|f\| = \Omega(\exp(-\Theta(k))). \]
\end{corollary}

\subsection{General activations}

It is not hard to see that the norm analysis of ReLU and sigmoid extends to any activation function for which a suitable lower bound on the Hermite coefficients holds, and which is either bounded or grows at a polynomial rate, so that under the standard Gaussian it behaves essentially identically to its truncated form. In particular, a lower bound of $\alpha^{-j}$ for any constant $\alpha < 4/e$ on the $j\th$ Hermite coefficient suffices to give $\|g\| \geq \exp(\Theta(k))$, by the same argument as in \cref{lem:g-relu-norm} and \cref{lem:g-sigmoid-norm}. This then suffices to give $\|f\| \geq \exp(-\Theta(k))$, as above.

In fact, even a very weak lower bound on $\|f\|$ yields \emph{some} superpolynomial bound on learning. Suppose we only had $\|f\| \geq 1/\exp(\exp(\Theta(k)))$, for instance. Then we can take $k = \log \log n$ and have $\|f\| \geq 1/\poly(n)$ and still obtain a lower bound of $n^{\log \log n} = n^{\omega(1)}$ (see \cref{thm:full-learning-bound}). Any lower bound on $\|f\|$ will be a function only of $k$, so a similar argument applies.

\section{SQ lower bound for real-valued functions proof}\label{sec:szorenyi-proof}

We give a self-contained variant of the elegant proof of \cite{szorenyi2009characterizing} for the reader's convenience. For simplicity, we include the $0$ function in our class $\cC$ --- this can only negligibly change the SDA, and it makes the core argument cleaner.

\begin{theorem}
		Let $D$ be a distribution on $X$, and let $\cC$ be a real-valued concept class over a domain $X$ such that $0 \in \cC$, and $\|c\|_D > \epsilon$ for all $c \in \cC, c \neq 0$. Consider any SQ learner that is allowed to make only inner product queries to an SQ oracle for the labeled distribution $D_c$ for some unknown $c \in \cC$. Let $d = \sda_D(\cC, \gamma)$. Then any such SQ learner needs at least $d/2$ queries of tolerance $\sqrt{\gamma}$ to learn $\cC$ up to $L_2$ error $\epsilon$.
	\end{theorem}
	
	\begin{proof}
	Consider the adversarial strategy where we respond to every query $h : X \to \R$ ($\|h\|_D \leq 1$) with 0. This corresponds to the true expectation if the target were the 0 function. By the norm lower bound, outputting any other $c$ would then mean $L_2$ error greater than $\epsilon$. Thus we must rule out all other $c \in \cC$.
	
	Let $\tau = \sqrt{\gamma}$. If $h_k$ is the $k\th$ query, let $S_k = \{c \in \cC \mid \inn{c, h_k}_D > \tau\}$ be the functions ruled out by our response of 0. (A similar argument will hold for $S_k' = \{c \in \cC \mid \inn{c, h_k}_D < -\tau\}$.) Let $\Phi = \inn{h_k, \sum_{c \in S_k} c}_D$. We claim that $\abs{S_k} \leq \abs{\cC} / d$. Suppose not. Then $\rho_D(S_k) \le \gamma$ by \cref{lem:SQ-dim}, and
	\begin{align*}
	    \Phi &\leq \|h_k\|_D \norm{\sum_{c \in S_k} c}_D \\
	    &\leq \sqrt{\sum_{c, c' \in S_k} \inn{c, c'}_D} \\
	    &= \sqrt{\abs{S_k}^2 \rho_{D}(S_k) } \\
	    &\leq \sqrt{\gamma} |S_k|,
	\end{align*} 
	contradicting the fact that  $\Phi > |S_k| \tau$ by definition of $S_k$. 
	
	Similarly $|S_k'| = |\{c \in \cC \mid \inn{c, h_k}_D < -\tau\}| \leq |\cC|/d$. Thus we rule out at most a $2/d$ fraction of functions with each query, and hence need at least $d/2$ queries to rule out all other possibilities.
\end{proof}
	
\end{document}